\documentclass[11pt]{amsart}
\usepackage{amsaddr}
\usepackage[margin={1.25in}]{geometry}
\usepackage{xstring}

\usepackage{amsmath,amsthm,wrapfig}
\usepackage[utf8]{inputenc} 
\usepackage{amsmath, amssymb,bm, cases, mathtools, thmtools}
\usepackage{verbatim}
\usepackage{graphicx}\graphicspath{{figures/}}
\usepackage{multicol}
\usepackage{tabularx}
\usepackage[usenames,dvipsnames]{xcolor}
\usepackage{mathrsfs} 
\usepackage{caption}
\usepackage{algorithm}
\usepackage{algorithmic}

\usepackage[%
    minnames=1,maxnames=99,maxcitenames=3,
    uniquelist=false,
    style=authoryear,%
    sortcites=true,
    doi=false,url=false,
    giveninits=true,
    hyperref,natbib,backend=biber]{biblatex}
\renewbibmacro{in:}{%
  \ifentrytype{article}{}{\printtext{\bibstring{in}\intitlepunct}}}
  
\bibliography{biblio}

\defbibenvironment{bibliography}
  {\trivlist}
  {\endtrivlist}
  {\item
   \printtext[labelnumberwidth]{%
     \printfield{labelprefix}%
      \printfield{labelnumber}}%
   \addspace}

\usepackage{datetime}

\DeclareMathAlphabet\EuRoman{U}{eur}{m}{n}
\SetMathAlphabet\EuRoman{bold}{U}{eur}{b}{n}

\usepackage[T1]{fontenc}
\usepackage{inconsolata}

\usepackage{url}
\usepackage{booktabs}
\usepackage{amsfonts}
\usepackage{nicefrac}
\usepackage{enumitem}

\usepackage{hyperref}

\usepackage[capitalize]{cleveref}

\crefname{lemma}{Lemma}{Lemmas}
\crefname{corollary}{Corollary}{Corollaries}
\crefname{theorem}{Theorem}{Theorems}

\makeatletter
\let\reftagform@=\tagform@
\def\tagform@#1{\maketag@@@{\ignorespaces\textcolor{gray}{(#1)}\unskip\@@italiccorr}}
\renewcommand{\eqref}[1]{\textup{\reftagform@{\ref{#1}}}}
\makeatother

\declaretheorem[style=plain,numberwithin=section,name=Theorem]{theorem}
\declaretheorem[style=plain,sibling=theorem,name=Lemma]{lemma}
\declaretheorem[style=plain,sibling=theorem,name=Proposition]{proposition}

\declaretheorem[style=plain,numberwithin=section,name=Claim]{claim}

\declaretheorem[style=definition,sibling=theorem,name=Definition]{definition}

\declaretheorem[style=remark,qed=$\triangleleft$,sibling=theorem,name=Remark]{remark}
\numberwithin{theorem}{section}

\usepackage{etoolbox}

\newcommand{\EE}{\mathbb{E}}

\newcommand{\NN}{\mathbb{N}}
\newcommand{\OO}{\mathbb{O}}
\newcommand{\PP}{\mathbb{P}}

\newcommand{\RR}{\mathbb{R}}

\newcommand{\Aa}{\mathcal{A}}

\newcommand{\Ff}{\mathcal{F}}
\newcommand{\Gg}{\mathcal{G}}
\newcommand{\Hh}{\mathcal{H}}
\newcommand{\Ii}{\mathcal{I}}

\newcommand{\Ss}{\mathcal{S}}

\newcommand{\Xx}{\mathcal{X}}
\newcommand{\Yy}{\mathcal{Y}}

\newcommand{\one}{\mathbf{1}}

\def\[#1\]{\begin{equation}\begin{aligned}#1\end{aligned}\end{equation}}
\def\*[#1\]{\begin{align*}#1\end{align*}}

\def\s[#1\s]{\small\begin{equation}\begin{aligned}#1\end{aligned}\end{equation}\normalsize}
\def\s*[#1\s]{\small\begin{align*}#1\end{align*}\normalsize}

\def\ss[#1\ss]{\scriptsize\begin{equation}\begin{aligned}#1\end{aligned}\end{equation}\normalsize}
\def\ss*[#1\ss]{\scriptsize\begin{align*}#1\end{align*}\normalsize}
\def\clap#1{\hbox to 0pt{\hss#1\hss}}

\newcommand{\lcrx}[4][{-1}]{ 
	\IfEq{#1}{-1}{\left #2 {{{{#3}}}} \right #4}{
	\IfEq{#1}{0}{#2 {{{{#3}}}} #4}{
	\IfEq{#1}{1}{\bigl #2 {{{{#3}}}} \bigr #4}{
	\IfEq{#1}{2}{\Bigl #2 {{{{#3}}}} \Bigr #4}{
	\IfEq{#1}{3}{\biggl #2 {{{{#3}}}} \biggr #4}{
	\IfEq{#1}{4}{\Biggl #2 {{{{#3}}}} \Biggr #4}{
    \GenericWarning{"4th argument to lcrx must be -1, 0, 1, 2, 3, or 4"}
    }}}}}}} %

\newcommand{\stk}[2]{\ensuremath{\stackrel{\text{#2}}{#1}}}

\newcommand{\ol}{\overline}

\makeatletter
\newcommand{\subalign}[1]{%
  \vcenter{%
    \Let@ \restore@math@cr \default@tag
    \baselineskip\fontdimen10 \scriptfont\tw@
    \advance\baselineskip\fontdimen12 \scriptfont\tw@
    \lineskip\thr@@\fontdimen8 \scriptfont\thr@@
    \lineskiplimit\lineskip
    \ialign{\hfil$\m@th\scriptstyle##$&$\m@th\scriptstyle{}##$\crcr
      #1\crcr
    }%
  }
}
\makeatother

\renewcommand{\Pr}{\mathbb{P}} %
\def\EE{\mathbb{E}} %

\newcommand{\given}{\mid}

\newcommand{\dist}{\sim}
\newcommand{\distiid}{\stk{\dist}{iid}}
\newcommand{\distind}{\stk{\dist}{ind}}

\newcommand{\andT}{\ \text{and}\ }

\newcommand{\ind}[1]{\one_{#1}} %

\def\multiset#1#2{\ensuremath{\left(\kern-.3em\left(\genfrac{}{}{0pt}{}{#1}{#2}\right)\kern-.3em\right)}}

\newcommand{\tr}{\mathrm{Tr}} %

\newcommand{\IMat}[1][{-1}]{\IfEq{#1}{-1}{I}{I_{#1}}}

\newcommand{\dee}{\mathrm{d}} %

\newcommand{\eqdist}{\stk={d}}

\DeclareMathOperator{\normaldist}{N}

\DeclareMathOperator{\unifdist}{Unif}

\newcommand{\rbra}[2][{-1}]{\lcrx[#1] ( {#2} ) }
\newcommand{\sbra}[2][{-1}]{\lcrx[#1] [ {#2} ] }

\newcommand{\abs}[2][{-1}]{\lcrx[#1] \vert {#2} \vert }
\newcommand{\set}[2][{-1}]{\lcrx[#1] \{ {#2} \}}

\newcommand{\ceil}[2][{-1}]{\lcrx[#1] \lceil {#2} \rceil}
\newcommand{\norm}[2][{-1}]{\lcrx[#1] \Vert {#2} \Vert}

\newcommand{\card}[2][{-1}]{\lcrx[#1] \vert {#2} \vert }

\newcommand{\Nats}{\NN}

\newcommand{\Reals}{\RR}

\newcommand{\PosReals}{\Reals_+}

\newcommand{\ProbMeasures}{\mathcal{M}_1}

\newcommand{\NNReals}{\PosReals}

\newcommand{\rangeO}[2][{-1}]{
	\IfEq{#1}{-1}{\set{0,\dots, #2}}{\set{#1,\dots, #2}}}

\newcommand{\range}[1]{[#1]}

\newcommand{\Union}{\bigcup}

\newcommand{\OrthProjs}{\OO\PP}

\newcommand{\cF}{\mathcal F}

\newcommand{\features}{\Reals^{\idim}}
\newcommand{\labels}{K}

\newcommand{\loss}{\ell}

\newcommand{\nats}{\mathbb{N}}

\renewcommand{\algorithmicrequire}{\textbf{Input:}}
\renewcommand{\algorithmicensure}{\textbf{Output:}}
\newcommand{\e}{\mathrm{e}}

\newcommand{\Dist}{\mathcal D}

\newcommand\optparen[1]{\ifthenelse{\equal{#1}{}}{}{(#1)}}
\newcommand{\RiskChar}{L}
\newcommand{\Risk}[2]{\RiskChar_{#1}\optparen{#2}}
\newcommand{\EmpRisk}[2]{\RiskChar_{#1}\optparen{#2}}

\newcommand{\dataspace}{Z}

\newcommand{\HS}{\mathcal H}

\newcommand{\defn}[1]{\emph{#1}}
\renewcommand{\ProbMeasures}[1]{\mathcal M_1(#1)}
\newcommand{\GibbsC}{\ProbMeasures{\HS}}

\renewcommand{\dataspace}{\mathcal S}
\newcommand{\cPr}[1]{\Pr^{#1}}

\newcommand{\hh}{\hat h}

\newcommand{\sGClong}{structural Glivenko--Cantelli}
\newcommand{\SGC}{SGC}
\newcommand{\SGCadj}{Structural}
\newcommand{\sGCadj}{structural}

\apptocmd{\sloppy}{\hbadness 10000\relax}{}{}

\usepackage[hide=false,setmargin=true,marginparwidth=1in]{marginalia}
\allowdisplaybreaks

\begin{document}

\title[In Defense of Uniform Convergence: Generalization via derandomization]{In Defense of Uniform Convergence: \\Generalization via derandomization with an application to interpolating predictors}

\author{Jeffrey Negrea}
\address{Dept.\ of Statistical Sciences, Univ.\ of Toronto; Vector Institute}
\author{Gintare Karolina Dziugaite}
\address{Element AI}
\author{Daniel M. Roy}
\address{Dept.\ of Statistical Sciences, Univ.\ of Toronto; Vector Institute}

\begin{abstract}
We propose to study the generalization error of a learned predictor $\hh$ in terms of that of a surrogate (potentially randomized) predictor that is coupled to $\hh$ and designed to trade empirical risk for control of generalization error.
In the case where $\hh$ interpolates the data,
it is interesting to consider theoretical surrogate classifiers that are partially derandomized or rerandomized, e.g., fit to the training data but with modified label noise.
We also show that replacing $\hh$ by its conditional distribution with respect to an arbitrary $\sigma$-field is a convenient way to derandomize.
We study two examples, inspired by the work of \citet{NK19c} and \citet{bartlett2019benign},
where the learned classifier $\hh$ interpolates the training data with high probability, has small risk, and, yet, does not belong to a nonrandom class with a tight uniform bound on two-sided generalization error.
At the same time, we bound the risk of $\hh$ in terms of surrogates constructed by conditioning and denoising, respectively, and shown to belong to nonrandom classes with uniformly small generalization error.
\end{abstract}

\maketitle

\renewcommand*{\thefootnote}{\fnsymbol{footnote}}
\footnotetext{We would like to thank Alexander Tsigler for pointing out an error in the statement of Lemma 5.3 and Theorem 5.4, caused by a missing hypothesis in one of our technical lemmas.
In these results,
we establish that the surrogate predictor, implicitly studied by \citet{bartlett2019benign} in the setting of overparameterized linear regression, belongs to a structural Glivenko--Cantelli (GC) class. We originally claimed that this was the case under conditions on the sequence of covariance matrices slightly weaker than the ``benign'' conditions introduced by \citeauthor{bartlett2019benign}. We also presented a new risk bound, under these slightly weaker conditions. 
In order to resolve the error, and given that much related work 
had appeared since our work first appeared, we chose a straightforward fix: 
we now establish the structural GC property under the same benign condition of \citeauthor{bartlett2019benign}. Using also the same results on sample covariance matrices, our expected risk bound now essentially matches the high probability bound of  \citeauthor{bartlett2019benign}, although our result is still proven by arguing about (structural) uniform convergence.
}
\renewcommand*{\thefootnote}{\arabic{footnote}}

\section{Introduction}
One of the central problems in learning theory is to explain the statistical performance of deep learning algorithms.
There is particular interest in explaining how overparameterized neural networks, trained by simple variants of stochastic gradient descent (SGD), simultaneously achieve low risk and zero empirical risk on benchmark datasets.
While certain naive explanations have been ruled out \citep{Rethinking17}, progress has been slow.

The bulk of recent work on this problem implicitly assumes the classifier learned by SGD belongs to a (potentially data-dependent) class for which there is a uniform and tight (two-sided) bound on the generalization error \citep{bartlett2017spectrally,neyshabur2017pac,golowich2017size,long2020size,wei2019data}.
After raising this observation,
\citet{NK19c} argue that this approach may be %
unable to explain performance observed in overparameterized models.
They argue this point by constructing a simple problem where an SGD-like algorithm learns a classifier that achieves low risk and zero empirical risk,
yet the learned classifier does not belong (even with high probability) to a class whose generalization error is uniformly small.

In this work, we initiate a response to \citet{NK19c} in defense of the utility of uniform convergence for understanding learning algorithms that obtain zero empirical risk.
We will use the term ``interpolating'' to refer to such learning algorithms and the corresponding learned hypotheses, borrowing the terminology used for functions that achieve zero mean squared error.
In problems involving interpolation, the complexity of the task (e.g., model complexity) often must increase with the dataset size in order for interpolation to be possible. This mirrors deep learning practice, where scientists will train larger, more complex models when presented with a larger dataset.
Since the complexity of the learning problems in question---and possibly even the sample spaces generating the data---change with the sample size, the traditional notions of uniform convergence (Glivenko--Cantelli classes) are not applicable.
Therefore we need to extend the concept of uniform convergence to the setting of sequences of learning problems of increasing complexity, which we do in \cref{sec:SGC} by defining the \emph{\sGClong{}} property.
Then in \cref{sec:decomp} we introduce a general approach to relate sequences of learning problems which are not \sGClong{} to ones that may be.
The basic idea is to introduce a surrogate learning algorithm that closely mimics the learning algorithm of interest, yet whose output belongs to a class (the surrogate hypothesis class) for which a uniform and vanishing bound on two-sided generalization error holds.

The observations of \citet{NK19c} relate to a number of other empirical learning phenomena that demand explanation.
One example is the phenomenon of double descent, brought to light by \citet{belkin2019reconciling,advani2017high,geiger2019jamming}.
The difficulty of explaining these double descent curves using standard uniform convergence arguments is a central theme of recent talks by Belkin. In a line of work by \citet{hastie2019surprises,mei2019generalization},
double descent was observed in unregularized, overparameterized linear regression.
\citet{bartlett2019benign} show that, for sequences of overparameterized linear regression tasks, the minimum norm interpolating solution to least squares will achieve asymptotically optimal risk with high probability given constraints on the covariate (feature) distribution.
In this setting, we show that no class containing the learned hypothesis with high probability can have a vanishing uniform bound on the absolute generalization error. In fact, such a bound cannot be representative of the risk of the learned hypothesis.
In \cref{sec:LinearRegression}, we show that the analysis of \citet{bartlett2019benign} may be viewed as introducing a surrogate predictor.
The surrogate in this case is the minimum-norm interpolating solution on the training data \emph{with label noise removed}.
To establish that the surrogate predictor belongs to a class with the \sGClong{} property, we rely on the concentration results for empirical covariance matrices by \citet{koltchinskii2014concentration}.
We combine a uniform bound based on the \sGClong{} property with other components of the analysis of \citet{bartlett2019benign} to obtain similar bounds on the expected risk of the minimum norm interpolating solution.

In \cref{sec:surr-cond}, we provide a relatively flexible recipe for constructing surrogate classifiers via probabilistic conditioning.
The approach produces a probability measure over hypotheses via retraining on data that is equal in distribution to the original training data but has been partially ``rerandomized''.
The approach effectively trades empirical risk for generalization error.
Lastly, in \cref{sec:hypercube}, we apply this recipe to an example, inspired by \citet{NK19c}, where an interpolating learning algorithm is constructed for which there is no \sGClong{} class containing the learned hypothesis with high probability.
In that example, we construct a surrogate by conditioning with respect to a specific $\sigma$-field.
We show the corresponding surrogate class is \sGClong{} and can be used to derive risk bounds for the learned classifier,
which exhibit a form of double descent.

\subsection{Contibutions}
In this work, we extend our theoretical understanding of generalization, by way of the following contributions:
\begin{enumerate}
  \item Defining the \sGClong{} property, a notion of uniform convergence for sequences of learning problems.

  \item Proposing to study generalization error of learning algorithms---including interpolating ones---in terms of surrogate hypotheses that may belong to \sGClong{} classes, even when the original hypotheses do not.

  \item Demonstrating that the hypothesis spaces corresponding to a sequence of unregularized, overparameterized linear regression tasks are not \sGClong{}, but that they can be analyzed by introducing a sequence of surrogates for which the surrogate hypothesis class is \sGClong{}. We further use this fact to provide bounds on the expected risk of the original sequence of tasks under the same hypotheses as \citet{bartlett2019benign}.

  \item Introducing a generic technique by which one may introduce surrogate learning algorithms via conditioning, which naturally trades empirical risk for generalization error relative to the original learning algorithm.

  \item Analyzing an example that distills the key features of an example in \citet{NK19c}, via a family of surrogates obtained from conditioning. We show that, while the original learning algorithm does not output hypotheses in a sequence of classes with the \sGClong{} property, discarding a few bits of information leads to one that does. We also show that bounds obtained via the surrogate learning algorithm exhibit a form of double descent.
\end{enumerate}

\section{Preliminaries}
Let $Z_1,\dots,Z_n$ be i.i.d.\ random elements in a space $\dataspace$ with common distribution $\Dist$.
Let $S=(Z_1,\dots,Z_n)$ represent the training sample.
Fix a loss function $\ell : \HS \times \dataspace \to \NNReals$ for a space $\HS$ of hypotheses.
Let $\GibbsC$ be the space of distributions on $\HS$. Note that $\HS$ can be embedded into $\GibbsC$ by the map $h \mapsto \delta_h$ taking a classifier to a Dirac measure degenerating on $\{{h}\}$.
For $Q \in \GibbsC$, the \defn{(average) loss} and \defn{risk} are defined to be
\*[
\loss(Q,z) &= \int \loss(h,z) Q(\dee h) ,
  & \Risk{\Dist}{Q} &= \int \loss(Q,z) \Dist(\dee z) .
\]
Let $\EmpRisk{S}{Q} = \Risk{\widehat \Dist_n}{Q}$ denote the \defn{empirical (average) risk},
where $\widehat \Dist_n = \frac 1 n \sum_{i=1}^{n} \delta_{Z_i}$ is the empirical distribution.
For $h \in \HS$, define $\Risk{\Dist}{h} = \Risk{\Dist}{\delta_h}$ and $\EmpRisk{S}{h} = \EmpRisk{S}{\delta_h}$.
Let $\hh$ or $\hh(S)$ be a random element in $\HS$, representing a learned classifier.

A hypothesis $h$ \defn{interpolates} a dataset $S$ with respect to a non-negative loss $\loss$ when $L_S(h) = 0$.
A learning algorithm $\hh(S)$ is (almost surely) \emph{interpolating} if $L_S(\hh(S)) = 0 $ a.s. (or equivalently $\EE L_S(\hh_S) =0$).
This extends our geometric intuition that a surface $h:\Reals^d\to \Reals$ interpolates points in $\set{(x_i,y_i)}_{i\in\range{n}}\Reals^{d}\times\Reals$ when $\rbra{h(x_i) - y_i}^2 = 0 $ for all $i\in\range{n}$. The surprising properties of interpolating classifiers are explored in \citet{belkin2019reconciling}. See also \citet{advani2017high,geiger2019jamming}.

\section{\SGCadj{} Uniform Convergence}
\label{sec:SGC}
\citet{NK19c} argue that uniform convergence does not explain generalization in several examples that are emblematic of the modern interpolating regime. In those examples, however,
the size of the learning problem varies with the cardinality of the training dataset.
The standard notion of uniform convergence (i.e., of Glivenko--Cantelli classes, etc.) is not normally defined in this setting.
In order to formalize the specific failure of ``uniform convergence'' in these sequences of learning problems,
we introduce a \emph{\sGCadj{}} version of the Glivenko--Cantelli property.

\newcommand{\upper}[1]{^{(#1)}}

\newcommand{\sgcDist}[1]{\Dist\upper{#1}}
\newcommand{\sgcInstances}[1]{\dataspace\upper{#1}}
\newcommand{\sgcSA}[1]{\Ff\upper{#1}}

\newcommand{\sgcPrSpace}[1]{\rbra{\sgcInstances{#1},\sgcSA{#1},\sgcDist{#1}}}

\newcommand{\sgcHS}[1]{\HS\upper{#1}}
\newcommand{\sgcSize}[1]{n_{#1}}

\begin{definition} Let $\set[0]{\sgcPrSpace{p}}_{p\in \Nats}$ be a sequence of probability spaces where $\sgcInstances{p}$ denotes the sample space, $\sgcSA{p}$ denotes the $\sigma$-field and $\sgcDist{p}$ denotes the probability measure. Let $\sgcHS{p}$ be a collection of measurable functions on $\sgcPrSpace{p}$ and let $\sgcSize{p}\in\Nats$ for all $p\in\Nats$.

Then $\sgcHS{\cdot}$ has the \emph{\sGCadj{} ($\sgcDist{\cdot},\sgcSize{(\cdot)}$)-Glivenko--Cantelli} property, denoted $(\sgcDist{\cdot},\sgcSize{(\cdot)})$-\SGC, if
\*[
  \lim_{p\to\infty}\EE \sbra{\sup_{h\in \sgcHS{p}} \abs{\sgcDist{p} h - \widehat{\sgcDist{p}}_{\sgcSize{p}}h} } = 0 ,
\]
where $P h = \int h(x) P(d x)$ and $\widehat{\sgcDist{p}}_{\sgcSize{p}}$ is the empirical distribution of an IID sample of size $\sgcSize{p}$ from $\sgcDist{p}$.
\end{definition}
It is this property which is made to fail in the examples presented by \citet{NK19c}.
When $\set{\sgcPrSpace{p}}_{p\in \Nats}$ and $\sgcHS{p}$ are constant and $\sgcSize{p}=p$, this reduces to the classical notion of Glivenko--Cantelli.

\begin{remark}[Relationship between PAC and nonuniform learning] PAC learnability and nonuniform learnability, as defined in \citep{shalev2014understanding}, can both be understood in terms of the \sGClong{} property.
That PAC learnability implies \sGClong{} follows from the equivalence of PAC learnability with the uniform Glivenko--Cantelli property.

To understand the nonuniform learnability of a class $\Hh$ in terms of the \sGClong{} property, recall the equivalence that $\Hh$ is nonuniformly learnable if and only if it is a countable union of VC classes---$\Hh = \Union_{j\in\nats} \Hh_j$ with $\Union_{j\in\range{p}}\Hh_j$ of finite VC dimension $d_p$.
Then, for any sequences $\delta_p \searrow 0$ and $\epsilon_p\searrow 0$, take $n_p \geq C_2\frac{d_p+\log(1/\delta_p)}{\epsilon_p^2}$ where $C_2$ is the universal constant appearing in \citep[Theorem~6.8, Item~1.]{shalev2014understanding}.
Taking $\set[0]{\sgcPrSpace{p}}_{p\in \Nats}$ to be constant, and $\Hh\upper{p} = \Union_{j\in\range{p}}\Hh_j$, it follows immediately  that
\*[
  \EE \sbra{\sup_{h\in \sgcHS{p}} \abs{\sgcDist{p} h - \widehat{\sgcDist{p}}_{\sgcSize{p}}h} }
    & \leq (1-\delta_p)\epsilon_p+\delta_p \to 0,
\]
and hence $\Hh\upper{\cdot}$ is $(\sgcDist{\cdot},\sgcSize{(\cdot)})$-\SGC. It would be reasonable in this case to say that $\Hh\upper{\cdot}$ is $(\sgcSize{(\cdot)})$-\SGC{} uniformly over data generating distributions.

The partitioning of the hypothesis space in synchronization with increasing sample size in this derivation is similar to the partitioning of the hypothesis space by sample size occurring in the structural risk minimization algorithm for nonuniform learning. The analysis above tells us that any ERM algorithm restricted to $\Hh\upper{p}$ when the sample size is $\sgcSize{p}$ will achieve low generalization error.
\end{remark}

\section{Decompositions of Generalization Error using Surrogate Classifiers}\label{sec:decomp}
We now describe how one may pass from bounding the generalization error of a learning algorithm to bounding the generalization error of a surrogate and controlling differences in the risk and empirical risk profiles between the original algorithm and the surrogate.

The following result is immediate from the linearity of expectation:
\begin{lemma}[Surrogate decomposition]
\label{surrogate}
For every random element $Q$ in $\GibbsC$,
\*[
 \EE[\Risk{\Dist}{\hh} - \EmpRisk{S}{\hh}]
&=  \smash{\EE[ \Risk{\Dist}{\hh} - \Risk{\Dist}{Q} ]}
\\&\;+ \smash{\EE[ \Risk{\Dist}{Q} - \EmpRisk{S}{Q} ]}
\\&\;+ \smash{\EE[ \EmpRisk{S}{Q} - \EmpRisk{S}{\hh} ] }
,
\]
provided the three expectations on the r.h.s.\ are finite.
\end{lemma}

This decomposition suggests that one can obtain a bound on the generalization error (and then the risk) of $\hh$ by bounding the three terms individually.
We interpret $Q$ here as a (possibly randomized) surrogate hypothesis that is coupled with $\hh$ via some information in the training algorithm and/or the training data.
The choice of $Q$ trades off one term for another.
In the particular case of a.s.\ interpolating classifiers (i.e., $\EE[L_S(\hh)] = 0$), one approach is to trade excess empirical risk,
$\EE[ \EmpRisk{S}{Q} - \EmpRisk{S}{\hh} ]$,
for less generalization error, $\EE[ \Risk{\Dist}{Q} - \EmpRisk{S}{Q} ]$.

One way to control generalization error is to show that $Q$ belongs to a nonrandom class
for which there holds a uniform and tight bound on generalization error.
\begin{proposition}[Bounded loss, two-sided control]\label{supbound}
Assume $\loss$ takes values in an interval of length $L$.
For every random element $Q$ in $\GibbsC$ and class $G \subseteq \ProbMeasures{\HS}$,
\*[
 & \EE [ \Risk{\Dist}{Q} - \EmpRisk{S}{Q} ] \\
 &\quad \le  L\,\Pr[Q \not\in G] + \EE \Big[ \sup_{P \in G}{ \abs {\Risk{\Dist}{P} - \EmpRisk{S}{P}} } \Big ].
\]
\end{proposition}
Just as we interpret $Q$ as a surrogate hypothesis that depends on the dataset $S$, we view $G$ in \cref{supbound} as a surrogate hypothesis class that contains the surrogate hypothesis with high probability.

The surrogate decomposition may be viewed as similar to a one-step covering argument, where the cover is given by the class of surrogate hypotheses, and the approximation error is given by $\EE[ \Risk{\Dist}{\hh} - \Risk{\Dist}{Q} ] + \EE[ \Risk{\Dist}{Q} - \EmpRisk{S}{Q}] $. In a typical one-step covering argument, the cover is chosen to be sufficiently fine as to have a uniformly small approximation error. The optimal cover density will vary with sample size so that approximation error vanishes as sample size increases. The key difference here is that we may not be able to control the approximation error uniformly or have any hope that it will vanish based on the covering induced by a surrogate. We will only attempt to uniformly control the cover given by the surrogate class. We then can rely on other techniques to handle the approximation error. This allows us to divide the objective of explaining generalization into a portion explained by uniform convergence and portion not explained by uniform convergence.

\section{Overparameterized Linear Regression}
\label{sec:LinearRegression}
Our first application of using surrogates and the \sGClong{} property to understand generalization error is inspired by the recent work of \citet{bartlett2019benign}. They determine necessary and sufficient conditions under which the \emph{minimum norm interpolating linear predictor} generalizes well in mean-squared error for random design linear regression in the overparameterized regime (i.e., more features than observations) for sub-Gaussian random designs with conditionally sub-Gaussian residuals. We chose overparameterized linear regression as a first example to present because of 1) the natural failure of any notion of uniform convergence to explain performance in the problem (in particular we show that the \sGClong{} property fails for any classes containing the learned hypotheses with high probability), and 2) recent work (such as by \citet{hastie2019surprises,mei2019generalization} and others) showing that double descent occurs in variants of this problem such as random feature regression.

We show that the decomposition used by \citet{bartlett2019benign} can be naturally related to a decomposition via a surrogate hypothesis, and that said sequence of classes of attainable surrogates is \sGClong{}. We use the uniform convergence of the surrogate to provide bounds on the expected generalization error. We consider only the Gaussian random design / Gaussian response case, but note that the results can be extended to the sub-Gaussian with minor modifications.
We provide bounds on the expected generalization error only, as the purpose of this example is to illustrate how uniform convergence of a surrogate may be used. 

\citet{bartlett2019benign} define a sequence of covariance matrices $\Sigma_n \in \Reals^{d_n\times d_n}$ to be \emph{benign} when
\*[
  \lim_{n\to \infty} \rbra{
    \sqrt{\frac{r_0(\Sigma_{n})}{n}}
    +\frac{k^*_n(\Sigma_n)}{n} +\frac{n}{R_{k^*_n(\Sigma_n)}(\Sigma_n)}} = 0,
\]
where (for some universal constant $b>0$ defined in \citep{bartlett2019benign})
\*[
  r_k(\Sigma_n)
    & = \frac{\sum_{i>k} \lambda_i(\Sigma_n)}{\lambda_{k+1}(\Sigma_n)}
      & R_k(\Sigma_n)
        & = \frac{\rbra{\sum_{i>k}\lambda_i(\Sigma_n)}^2}{\sum_{i>k}\lambda_i^2(\Sigma_n)},
\]
$\set{\lambda_i(\Sigma_n)}_{i\in \range{d}}$ are the eigenvalues of $\Sigma$ in decreasing order (with multiplicity), and
\*[
  k^*_n(\Sigma_n)
    & = \min\set{k\geq 0: r_k(\Sigma_n)\geq b n}.
\]
Their work shows that it is sufficient that $\Sigma_n$ be \emph{benign} in order to guarantee that $L_D(\hat\beta)\to \sigma^2$. They also show that it is necessary that
\*[
  \lim_{n\to \infty} \rbra{\frac{k^*_n}{n} +\frac{n}{R_{k^*_n(\Sigma_n)}(\Sigma_n)}} = 0
\]
in order for $\EE L_D(\hat\beta) \to \sigma^2$. Note that the risk of the true coefficient vector, $\beta$ (which is also the global optimizer of the risk) is $\sigma^2$, and so we expect the risk of any estimator to be at least that large.

\subsection{Construction}
Let $X_i\distiid\normaldist_{1\times d}(0,\Sigma_n)$ be random row vectors with non-singular $d\times d$ feature covariance matrix $\Sigma_n$ for $i\in\set{1,\dots,n}$. Let $X = (X_1',\dots X_n')'$ be the corresponding $n\times d$ random design matrix. Let $(Y_i\given X)\distind \normaldist(X_i\beta_n,\sigma^2)$ and $Y = (Y_1,\dots Y_n)'$ be the responses and response vector respectively. Let $Z = Y-X\beta_n$ be the residual vector. The loss function will be squared error $\loss(\beta, (x,y)) = (x\beta -y)^2$. We want to understand the generalization performance of the
\emph{minimum norm interpolating linear predictor} for $(X,Y)$, $\hat \beta(X,Y) = (X'X)^+X'Y$ where $A^+$ denotes the Moore--Penrose pseudo-inverse for $A$.

\subsection{Failure of uniform convergence for this problem}
\begin{lemma}[Failure of uniform convergence for overparameterized linear regression]
  \label{lem:lr-failureUC}
  There is no sequence of measurable sets $\set{A_n}_{n\in\Nats}$ such that $\Pr((X,Y)\in A_n)>2/3$ for all $n\in\Nats$ and for which
  \s*[
    \limsup_{n\to\infty}\EE \sup_{ (\tilde X,\tilde Y) \in A_n} \abs{L_D(\hat\beta(\tilde X,\tilde Y)) - L_S(\hat\beta(\tilde X,\tilde Y))} \leq \frac{3}{2} L_D(\beta) .
  \s]
\end{lemma}
The proof of this result is found in \cref{apx:Proofs2}.

\subsection{Introducing a surrogate}
We will consider the surrogate given by the minimum norm interpolating predictor for the training data with label noise removed. Mathematically, this can be defined by taking $\hat\beta_0 = (X'X)^+X'X\beta$. Notice that $\hat \beta_0 = P(X)\beta$ where $P(X)$ is the projection onto the row-space of $X$.

The surrogate decomposition of the generalization error for $\hat\beta$ is given in the following lemma.
\begin{lemma}[Surrogate decomposition of $\hat\beta$]
  \label{lem:LR-sur-decomp}
  \s*[
    & L_D(\hat \beta ) - L_S(\hat \beta) \\
      &\quad  = (L_S(\hat \beta_0)-L_S(\hat\beta)) + (L_D(\hat \beta) - L_D(\hat \beta_0))\\ &\qquad\quad + (L_D(\hat \beta_0) - L_S(\hat \beta_0)) ,
  \s]
  with
  \s*[
    L_S(\hat \beta_0) - L_S(\hat\beta)
      & = \frac{1}{n}\norm{Z}^2 \\
    L_D(\hat \beta) - L_D(\hat \beta_0)
      & = \tr(X(X'X)^+\Sigma_n(X'X)^+X'Z Z')\\
    L_D(\hat \beta_0) - L_S(\hat \beta_0)
      & = \sigma^2-\frac{\norm{Z}^2}{n} + \beta_n'P(X)^\perp \Sigma_n P(X)^\perp \beta_n
  \s]
\end{lemma}
The proof appears in \cref{apx:Proofs2}. Note that $(L_D(\hat \beta) - L_D(\hat \beta_0))$ and $(L_S(\hat \beta_0) - L_S(\hat\beta))+(L_D(\hat \beta_0) - L_S(\hat \beta_0))$ are exactly the terms which \citet{bartlett2019benign} choose to bound separately. They, however, made this decision using the bias--variance decomposition of the generalization error rather than arriving at the decomposition because it arose from a choice of surrogate.

\begin{lemma}[The sequence of surrogate hypothesis classes is \SGC{}]
  \label{lem:lr-sgcsurr}
The sequence of implied surrogate hypothesis classes, $\set[0]{\hat\beta_0(S): S\in\sgcInstances{n}}_{n\in\Nats}$ is $(\sgcDist{n},n)$-\SGC{} when $\set{\Sigma_n}_{n\in\Nats}$ is benign and $\set{\norm{\beta_n}^2 \norm{\Sigma_n}}_{n\in\Nats}$ is bounded. Quantitatively, for a universal constant $C>0$,
\*[
  & \EE\sup_{(X_0,Y_0)\in\Reals^{n\times d}\times \Reals^n} \abs{L_D(\hat\beta_0(X_0,Y_0)) - L_S(\hat\beta_0(X_0,Y_0))} \\
    &\qquad \leq C\frac{\sigma^2 +  \norm{\beta_n}^2\norm{\Sigma_n} \max\rbra{\sqrt{r_0(\Sigma_n)}, r_0(\Sigma_n)/\sqrt{n}}}{\sqrt{n}}
\]
\end{lemma}
The proof appears in \cref{apx:Proofs2}.
Combining \cref{lem:lr-sgcsurr} with \citep[Lemma 11]{bartlett2019benign} (which controls $L_D(\hat \beta) - L_D(\hat\beta_0)$), we get the following bound on the expected generalization error.
\begin{theorem}[Expected risk bound for overparameterized linear regression]
  \label{thm:lr-exp-risk-bd}
  For some universal constant $C,c>0$,
  \*[
    \EE L_D(\hat\beta)
      & \leq \sigma^2 + C\frac{\sigma^2 +  \norm{\beta_n}^2\norm{\Sigma_n} \max\rbra{\sqrt{r_0(\Sigma_n)}, r_0(\Sigma_n)/\sqrt{n}}}{\sqrt{n}} \\
      &\qquad+ c\sigma^2\rbra{\frac{k^*_n}{n} +\frac{n}{R_{k^*_n}(\Sigma_n)}}
  \]
  In particular, if $\set{\Sigma_n}_{n\in\Nats}$ is benign and $\set{\norm{\beta_n}^2\norm{\Sigma_n}}_{n\in\Nats}$ is bounded then $\EE L_D(\hat\beta)\to \sigma^2$.
\end{theorem}
The result is similar to what one would obtain by converting the high-probability bound of \citet{bartlett2019benign} into a bound in expectation. Our approach highlights the role of uniform convergence via a surrogate in this problem. It is note-worthy that, when viewing the surrogate class as a type of one-step covering as discussed in the comments after \cref{supbound}, the approximation error component of the surrogate decomposition does not vanish in this case. Instead, it tends to $\sigma^2$ when the covariance matrices are  benign, and may have more erratic behaviour otherwise.

\section{Constructing Surrogates by Conditioning}
\label{sec:surr-cond}
In the overparameterized linear regression example, the surrogate obtained by training on ``de-label-noised'' data allowed us to construct meaningful generalization bounds. For a generic learning problem, however, there may be no notion of label noise or such an approach may not prove useful. This leads us to seek natural constructions of surrogates in less structured problems.

One generic way to introduce such a surrogate is by conditioning.
Let $\cPr{\cF}$ denote the conditional probability operator given a $\sigma$-field $\cF$ (or a random variable), taking an event to its conditional probability. For a random variable $\psi$, let $\cPr{\cF}[\psi]$ denote the conditional distribution of $\psi$ given $\cF$.

\begin{lemma}[Derandomization via conditioning]
\label{conditioning}\label{condequiv}
Let $\cF$ be a $\sigma$-field on (some possible extension of) the underlying probability space upon which $S$ and $\hh$ are defined. Let $Q = \cPr{\cF}[\hh]$.
Then $\EE[\Risk{\Dist}{\hh}- \Risk{\Dist}{Q}] = 0$.
\end{lemma}
The following result is then immediate by \cref{surrogate,condequiv}.
\begin{lemma}[Surrogate decomposition by conditioning]
\label{surrogatecond}
Let $\cF$ and $Q$ be as in \cref{conditioning}.
Then
\*[
\EE[\Risk{\Dist}{\hh} - \EmpRisk{S}{\hh}]
  & = \EE[ \EmpRisk{S}{Q} - \EmpRisk{S}{\hh} ] \\
      &\ + \EE[ \Risk{\Dist}{Q} - \EmpRisk{S}{Q} ].
\]
If $\hh$ is a.s.\ interpolating (i.e., $\EE[\EmpRisk{S}{\hh}] = 0$),
then
\*[
\EE[\Risk{\Dist}{\hh}]
&= \EE[ \EmpRisk{S}{Q}]
      + \EE[ \Risk{\Dist}{Q} - \EmpRisk{S}{Q} ].
\]
\end{lemma}
Every conditional distribution $Q=\cPr{\cF}[\hh]$ represents a \defn{derandomization} of $\hh$:
i.e., by the definition of conditioning, $\hh$ has equal or greater dependence on the data $S$ than $Q$.
There are other ways to achieve derandomization
rather than conditioning. However, they may require one to obtain some explicit control on the risk difference,
$\EE[ \Risk{\Dist}{\hh} - \Risk{\Dist}{Q} ]$.

Informally, if $\hh$ interpolates (or more generally overfits),
we would expect a derandomized classifier to have excess empirical risk, yet lower generalization error.

Finally,
it is important to understand how tautologies can arise from this perspective.
If $Q$ is a.s.\ nonrandom (corresponding, e.g., to conditioning on the trivial $\sigma$-algebra),
then $Q=\Pr[\hh]$ a.s., i.e., $Q$ is the distribution of $\hh$.
In this case, $\EE\EmpRisk{S}{Q} = \EE \Risk\Dist{Q} = \EE \Risk\Dist{\hh}$, and we obtain the tautology
\*[
\EE[\Risk{\Dist}{\hh} - \EmpRisk{S}{\hh}]
&=
    \EE[ \Risk{\Dist}{\hh} - \Risk{\Dist}{Q} ] \\
      &\ + \EE[ \Risk{\Dist}{Q} - \EmpRisk{S}{Q} ] \\
      &\ + \EE[ \EmpRisk{S}{Q} - \EmpRisk{S}{\hh} ] \\
& = 0 + \EE[\Risk{\Dist}{\hh} - \EmpRisk{S}{\hh}] +0.
\]
For this extreme example, $Q$ belongs to the singleton class $\{{\Pr[\hh]}\}$,
which exhibits ``uniform convergence'' trivially.
On the other end of the spectrum, if $Q = \delta_{\hh}$, i.e., we condition on $\cF = \sigma(S)$,
then we get an equally tautological statement from the decomposition.
The idea behind introducing the surrogate classifier $Q$ is that it allows one to conceptually interpolate between these two tautological end points in order to find a (non-tautological) bound on the generalization error of a learning algorithm.
\section{Hypercube classifier}
\label{sec:hypercube}
\newcommand{\fs}{f^*}
\newcommand{\fh}[1]{\hat f_{#1}}
\newcommand{\tps}{\{0,1\}}

The following example is inspired by theoretical and empirical work by \citet{NK19c}.
Like in their work modelling SGD, we describe an example of a low-risk learned classifier, $\hh$,
such that there is no nonrandom class containing $\hh$ almost surely for which one may establish a uniform and nonvacuous bound on generalization error.
Using \cref{surrogatecond,supbound}, we show that a derandomization of $\hh$, obtained by conditioning on an explicit $\sigma$-field $\cF$,
yields a tight generalization bound based on uniform convergence of the surrogate.

In this section, we first construct the learning problem we will address.
Second, we show that the \sGClong{} property fails on this example even though is has low generalization error.
Lastly, we introduce our surrogate learning algorithm, show that it has similar empirical and test performance to the original algorithm, verify that the surrogate has the \sGClong{} property, and finally use this to establish a generalization bound for the original learning algorithms.

\subsection{Construction}
\renewcommand\features{\Xx}
\renewcommand\labels{\Yy}
\newcommand\instances{\Ss}
Let $d \gg 1$ index the dimensionality of the feature space and our sequence of learning problems, and let $n_d$ be the sample size for learning problem with index $d$.
Let $\features\upper{d} = \set{0,1}^{2 d }$ be the feature space and $\labels = \set{0,1}$ be the label space, and let $\sgcInstances{d}= \features\upper{d} \times \labels$.
Let $\fs_d: \features \to \labels$ be given by
\*[
\fs_d(x) = \begin{cases}
1, & \norm{X}_1 \le d, \\
0, & \text{otherwise.}
\end{cases}
\]
For $x \in \features\upper{d}$, note that $\fs_d(x) = 1- \fs_d(1-x)$.
Let $\sgcDist{d}$ be the distribution of $(X,\fs_d(X))$, where $X\dist\unifdist\rbra{\features\upper{d}}$.
Let $S = (Z_1,\dots,Z_{\sgcSize{d}}) \sim ({\sgcDist{d}})^{\sgcSize{d}}$ where $Z_i=(X_i,\fs_d(X_i))$.
Let $\fh{d}$ be the random element in $\tps^{2 d } \to \tps$ given by
\*[
\fh{d}\upper{S}(x) = \begin{cases}
1-\fs_d(x), & x\not\in S \andT 1-x\in S\\
\fs_d(x), & \text{otherwise}
\end{cases}.
\]
Let $\bar Z_i = (1-X_i,1-Y_i)$ and $\bar S = (\bar Z_1,\dots,\bar Z_n)$. We refer to pairs $(Z_i, \bar Z_i)$ as antipodes.
Our learning algorithm, $\Aa_d:S\to \fh{d}\upper{S}$, only makes a classification error when a test point was not in the training set, but its antipode was.

\subsection{Failure of uniform convergence for this problem}
First, we will note that at every problem size, $d$, the VC dimension of the collection of accessible decision rules is at least as large as the training dataset. We will not use this fact again, but it does highlight the apparent complexity of the learning problem.
\begin{proposition}[VC theory not applicable]
For $\sgcSize{d}\leq 2^{2 d-1}$, $\sgcHS{d} = \set[0]{\fh{d}^{(S)}: S \in (\sgcInstances{d})^{\sgcSize{d}}}$ has VC-dimension at least $\sgcSize{d}$.
\end{proposition}
\begin{proof}
For $\sgcSize{d}\leq 2^{2 d-1}$, any set of features $(X_1,\dots X_{\sgcSize{d}})$ of size $\sgcSize{d}$ with no antipodal points and no repeated points can be shattered by the
subcollection of $\sgcHS{d}$  given by $\set{\fh{d}^{(S)}: S\in\prod_{i\in\range{\sgcSize{d}}}\set{Z_i,\bar Z_i}}$.
\end{proof}
Next, notice that this algorithm never makes an error on the training data.
\begin{lemma}[$\fh{d}$ is interpolating.]
$\EmpRisk{S}{\fh{d}} = 0$ a.s.
\end{lemma}

Furthermore, by construction, the learning algorithm cannot return a classifier with high risk, no matter the training data observed, as long as $\sgcSize{d}\in o(2^{2 d })$.
\begin{lemma}[$\fh{d}$ has small risk]\label{lem:hypercube-small-risk}
$\Risk{\sgcDist{d}}{h} \leq \sgcSize{d} 2^{-2 d }$ for all $h\in\sgcHS{d}=\set[0]{\fh{d}^{(S)}: S \in (\sgcInstances{d})^{\sgcSize{d}}}$, and hence
\*[
  \Risk{\Dist}{{\fh{d}}} - \EmpRisk{S}{{\fh{d}}} & \leq \sgcSize{d} 2^{-2 d } \quad a.s.
\]
\end{lemma}
The proof appears in \cref{apx:Proofs}.
The following result demonstrates that uniform convergence (of a class containing $\hh$) \emph{does not} explain the risk.
The argument mirrors that of \citet{NK19c}.
\begin{theorem}[$\sgcHS{\cdot}$ is not \SGC]\label{thm:not-SGC} If $\sgcSize{d}\in o(2^{d})$ then $\set{\loss\circ \sgcHS{d}}_{d\in\Nats}$ is \emph{not} $(\sgcDist{\cdot},\sgcSize{(\cdot)})$-\SGC, in fact
\*[
    \EE \sup_{h\in \HS} \abs{\Risk{\Dist}{h} - \EmpRisk{S}{h}}
      & = 1 - O(n^2 2^{-2 d }).
\]
\end{theorem}
The proof appears in \cref{apx:Proofs}.
In this example, a generalization error bound was tractable because $\Risk\Dist{h}$ was readily bounded for all $h$, despite the fact that uniform convergence failed for $\sgcHS{\cdot}$.
One may then ask ``how many bits of information do we need to forget about our training data in order for the sequence surrogate hypothesis class obtained by conditioning is \sGClong?''

\subsection{Introducing a surrogate classifier}
Let $k_d \le 2 d $.
Let $\pi_{k_d} : \tps^{2 d } \to \tps^{2 d }$ satisfy
\*[
\pi_{k_d}(x_1,\dots,x_{2 d })_j = \begin{cases}
0, & j \le k_d, \\
x_j, & \text{otherwise.}
\end{cases}
\]
That is, $\pi_{k_d}(x)$ zeros out the first ${k_d}$ entries of $x$. Now, let $\pi_{k_d}(S) = (\pi_{k_d}(X_1),\dots,\pi_{k_d}(X_{\sgcSize{d}}))$, $\Gg\upper{d} = \sigma(\pi_{k_d}(S))$, and put $Q(S) = \cPr{\Gg\upper{d} }[\fh{d}^{(S)}]$.
$Q(S)$ is a Gibbs classifier that is learned from the data, but is less coupled with the data than $\fh{d}^{(S)}$. Intuitively, conditioning our learned classifier on $\Gg\upper{d}$ can be interpreted as redrawing the first $k$ features and the labels associated with the training data independently for each new test point, holding the last $2 d -k$ features of each training point fixed.
Since $Q(S)$ is $\sigma(\pi_{k_d}(S))$-measurable, then for distinct datasets $S$ and $S'$ with $\pi_{k_d}(S) = \pi_{k_d}(S')$, we have $Q(S) = Q(S')$.

The map $S\mapsto Q(S)$ is the \emph{surrogate learning algorithm} and $Q(S)$ is the \emph{surrogate classifier}. Note that in this example, our chosen surrogate learning algorithm returns a Gibbs classifier, while the original algorithm returns a deterministic classification rule.
When the argument $S$ of $Q$ is omitted then it is assumed to be the training dataset, $S$.

We first evaluate the risk and empirical risk of our surrogate classifier.

\begin{lemma}[Risk and empirical risk of the surrogate $Q$]\label{lem:Q-risk}
  The following all hold almost surely
  \s*[
    \EmpRisk{S}{Q}
      & \leq \frac{2^{-k_d}(1- 2^{-k_d})}{\sgcSize{d}} \\
      &\qquad\quad \times\card{\set{(i,j) :  X_i[k+1 : 2 d ] = \ol{X_j[k+1:2 d ]}}}\\
    \EmpRisk{\ol S}{Q}
      & \leq \frac{2^{-k}}{\sgcSize{d}} \card{\set{(i,j) : X_i[k+1 : 2 d ] = {X_j[k+1:2 d ]}}} \andT\\
    \Risk{\Dist}{Q}
      & \le \sgcSize{d} 2^{-2 d }.
  \s]
Furthermore,
\*[
\EE \EmpRisk{S}{Q}
  & \leq (\sgcSize{d}-1)2^{-2 d } (1-2^{-k_d}), \\
\EE \EmpRisk{\ol S}{Q}
  & \leq 2^{-k_d}+(\sgcSize{d}-1)2^{-2 d } , \andT \\
\EE \Risk{\Dist}{Q}
    & \le \sgcSize{d} 2^{-2 d }  .
\]
\end{lemma}
The proof appears in \cref{apx:Proofs}.
As was foreshadowed in \cref{sec:decomp}, we have increased the empirical risk by replacing $\fh{d}$ with $Q$.
However, at the same time, we have dramatically lowered the empirical risk on the adversarial (antipodal) dataset, and not affected the true risk at all.
In fact we are able to trade off empirical risk on the training data with worst case risk on an adversarial dataset explicitly by varying the parameter $k_d$.
Even a small amount of re-randomization in the surrogate ($k_d$ small) can yield very tight control on the adversarial empirical risk.
In this example, the adversarial empirical risk decreases exponentially fast in the number of bits of information lost per example.
This allows us to demonstrate that the sequence of surrogate hypothesis classes is \sGClong{}.

\newcommand{\sgcSHS}[1]{\Ii\upper{d}}

\begin{lemma}[The surrogate is \SGC{}]\label{lem:Q-sgc}
  Consider the sequence of surrogate hypothesis classes given by $\sgcSHS{d} = \set[0]{Q(S): S \in\sgcInstances{d}} $.
  If $\sgcSize{d}\in o(2^{k_d}/\sqrt{d})$ then $\set[0]{\loss\circ \sgcSHS{d}}_{d\in\Nats}$ is $(\sgcDist{\cdot},\sgcSize{(\cdot)})$-\SGC.
  In particular, when $\sgcSize{d}\in o(2^{2 d })$ and $k_d \geq \ceil[0]{(1+\epsilon)\log_2(\sgcSize{d})+\log_2(d)/2}$ then we have $\set[0]{\loss\circ \sgcSHS{d}}_{d\in\Nats}$ is $(\sgcDist{\cdot},\sgcSize{(\cdot)})$-\SGC.
  If $\log(\sgcSize{d}) \in \Omega(d)$ this is only possible if $k_d \in \Omega(d)$.
  If $\log(\sgcSize{d}) \in o(d)$ then this is possible even when $k_d \in o(d)$
\end{lemma}
The proof appears in \cref{apx:Proofs}. Note that the restrictions upon $k_d$ we provide may be a product of our particular approach to bounding the Rademacher complexity (Massart's Lemma). A more refined approach may yield looser restrictions on $k_d$.
Since the surrogate behaves similarly on training and test data to the original learning algorithm, and since the sequence of classes of achievable surrogates is \SGC, we can establish a generalization error bound for the original learning algorithm using the uniform convergence of the surrogate.

\begin{theorem}[Bounding generalization error via an \SGC{} surrogate]
  \label{thm:example1-main-bound}
We have the following bound on the generalization error:
  \s*[
    & \EE[\Risk{\Dist}{\fh{d}} - \EmpRisk{S}{\fh{d}}] \\
    &\quad \leq (\sgcSize{d}-1)2^{-2 d } + 2\sqrt{\log(2)}\sgcSize{d}^{1/2}((2 d -k_d)\sgcSize{d}+1)^{1/2}2^{-k_d} .
  \s]
If $\sgcSize{d}\in o\rbra{2^{2 d }}$, then for any choice of $\set{k_d}_{d\in\NN}$ with  $\lim_{d\to\infty}\rbra{k_d-\log_2(\sgcSize{d})-\log_2(d)/2}=\infty$, our surrogate witnesses that the generalization error vanishes
\*[\EE[\Risk{\Dist}{\fh{d}} - \EmpRisk{S}{\fh{d}}] \to 0 . \]
\end{theorem}
The proof appears in \cref{apx:Proofs}
\subsection{Interpreting the derandomization bound}
We visualize the dependence of our generalization error bound on $k$, $n$ and $d$ in \cref{fig:example1-bounds-plot1}. Notice that the bound decays very rapidly in the proportion of randomness removed by conditioning, $k/2 d $. When the learning problem is sufficiently complex, even for larger sample size, a small proportion of derandomization leads to a strong control of the generalization error.
\begin{figure}[!t]
  \includegraphics[width=.6\textwidth]{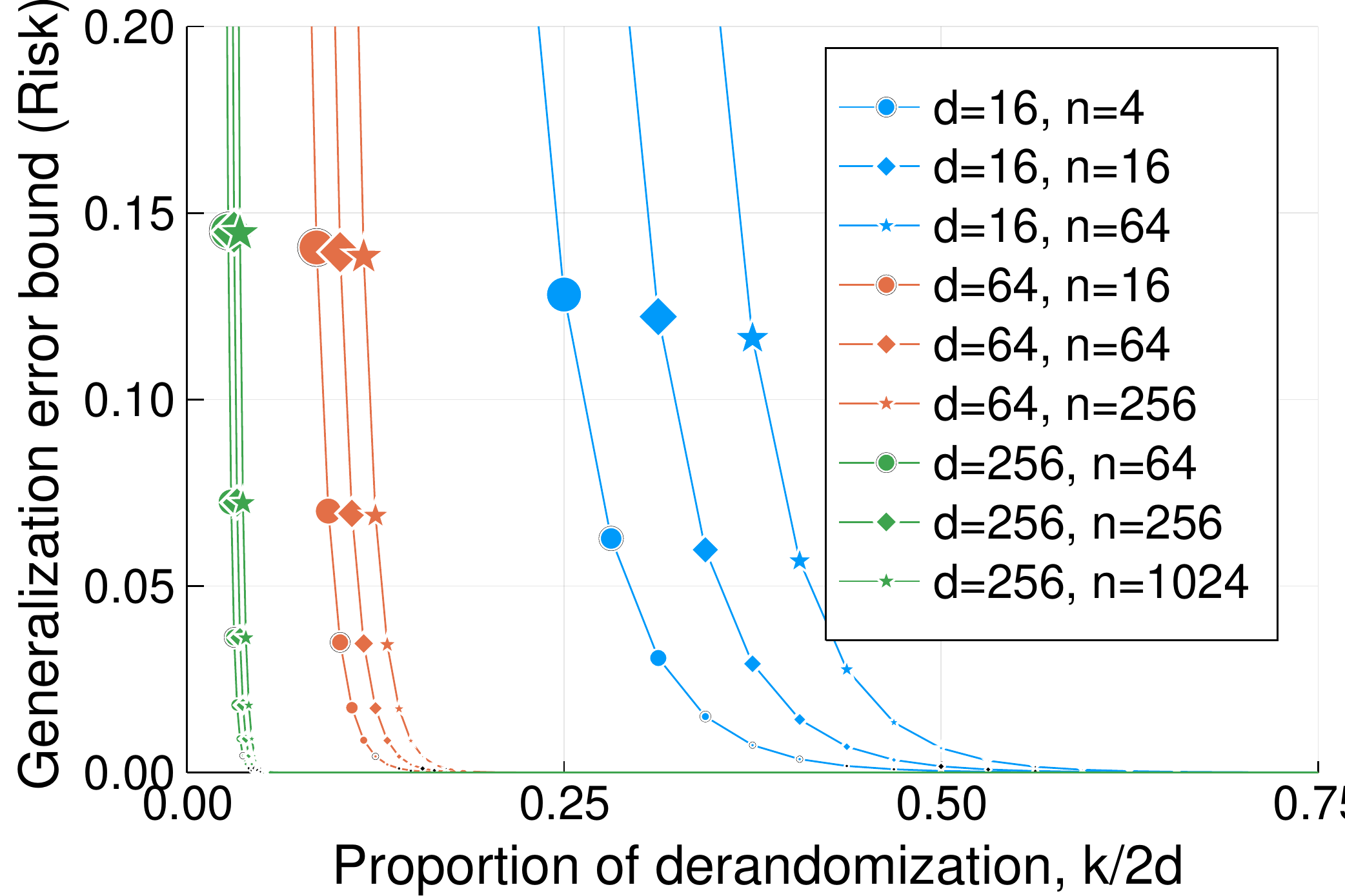}
  \caption{Visualizing the bound of \cref{thm:example1-main-bound}.}
  \label{fig:example1-bounds-plot1}
\end{figure}

\subsubsection{Relationship with Double Descent}
The bound produced by derandomization also exhibits a form of double descent.
The rising left branch of each curve in \cref{fig:example1-doubledescent-plot1} is a bound on the generalization error based on uniform convergence of a VC-class containing the learned predictor, where the VC dimension is bounded by $\min(n,2 d -1)$.
For $d\ll n$ this gives nonvacuous bounds for the low dimensional setting.
Since the classifier is interpolating almost surely, even in the low complexity setting (low dimension $d$) the first descent begins at $0$ and only shows the rising component as complexity increases.
The right branch of each curve in \cref{fig:example1-doubledescent-plot1}, based on the bound of \cref{thm:example1-main-bound}, shows the second descent in the high dimensional setting --- for sufficiently complex models one may bound the generalization error of learned classifier via uniform convergence of a suitable derandomized classifier. \emph{The bound obtained via derandomization is nonvacuous in the region of the second descent, exactly where standard uniform convergence techniques based on VC theory would give vacuous bounds.}
\begin{figure}[!t]
  \includegraphics[width=.6\textwidth]{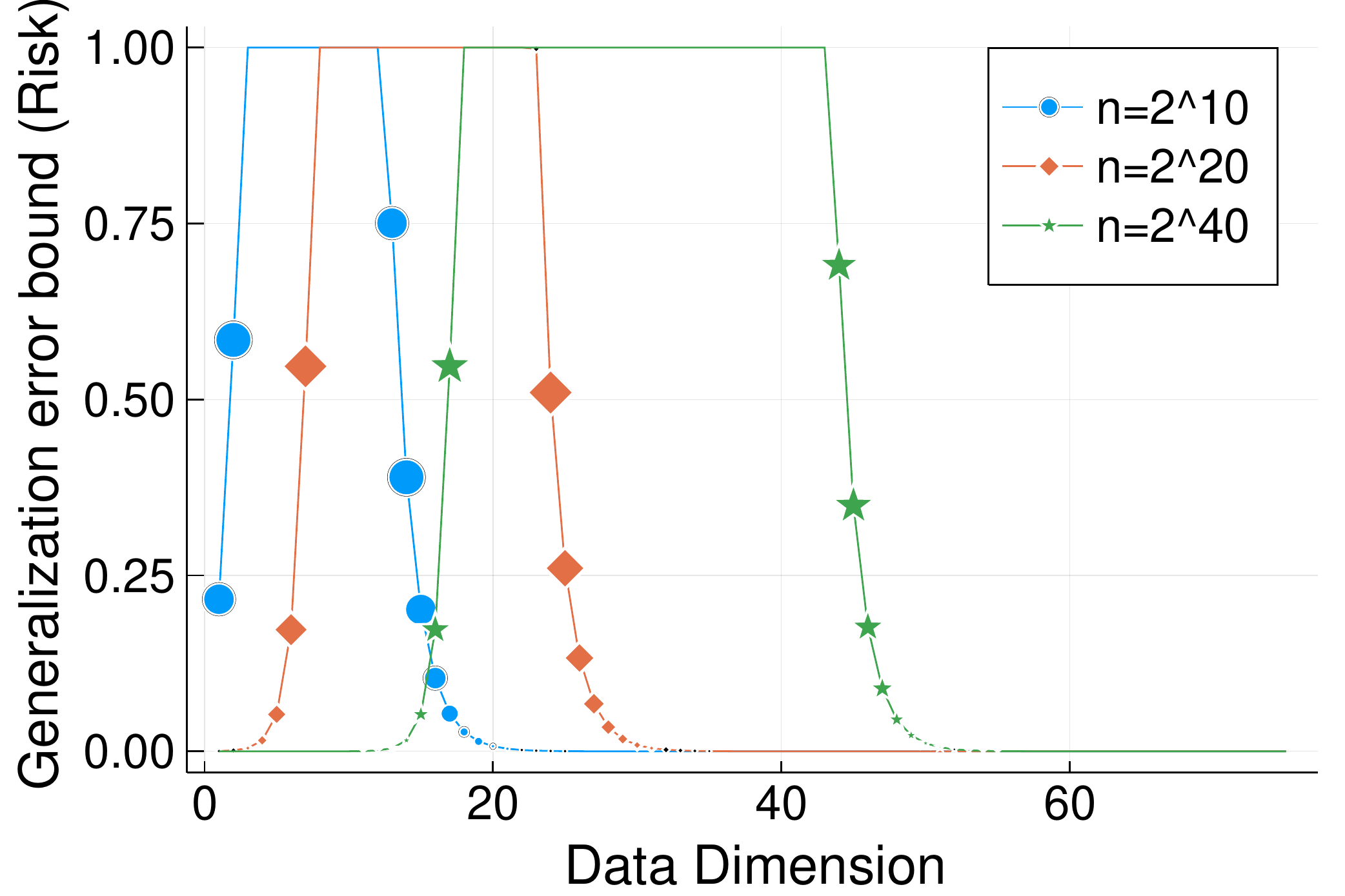}
  \caption{Double descent in the bound of \cref{thm:example1-main-bound}.}
  \label{fig:example1-doubledescent-plot1}
\end{figure}

\printbibliography

\appendix

\section{Proofs for hypercube classifier}
\label{apx:Proofs}
\begin{proof}[Proof of \cref{lem:hypercube-small-risk}]
  \*[
    \Risk{\sgcDist{d}}{\fh{d}^{(S)}}
      & \leq \frac{1}{2^{2 d }} \sum_{x \in \set{0,1}^{2 d }} \ind{x \in \bar S_X}
      \leq \sgcSize{d} 2^{-2 d } .
  \]
\end{proof}

\begin{proof}[Proof of \cref{thm:not-SGC}]
We require the following claims and lemmas.
  \begin{claim}
  $S \eqdist \bar S$.
  \end{claim}

  \begin{lemma}
  Let $E$ be the event that $\exists i,\ \exists j, \  X_i=1-X_j$.
  Then $\Pr E \leq \binom{n}{2} 2^{-2 d }$, so $\Pr E \to 0$ as $d \to \infty$ as long as $\sgcSize{d}\in o(2^d)$.
  \end{lemma}
  \begin{proof}
    \*[\Pr E
      & = \Pr(\exists i,\ \exists j, \  X_i=1-X_j)
       \leq \sum_{i\neq j } \Pr(X_i=1-X_j)
       = \binom{\sgcSize{d}}{2} 2^{-2 d }
    \]
  \end{proof}

  \begin{lemma}
  $\EmpRisk{\bar S}{\fh{d}} = 1$ on $E^c$.
\end{lemma}
  \begin{proof}
    \*[
      \EmpRisk{\bar S}{\fh{d}}
        & = \frac{1}{\sgcSize{d}} \sum_{i=1}^{\sgcSize{d}} \ind{\bar X_i \in \bar S_X}\ind{\bar X_i \not\in S_X}
        =\frac{1}{\sgcSize{d}} \sum_{i=1}^{\sgcSize{d}} \ind{\bar X_i \not\in S_X}
        \geq \ind{E^c} \frac{1}{\sgcSize{d}}\sum_{i=1}^{\sgcSize{d}} 1
         = \ind{E^c}
    \]
  \end{proof}
  Now, returning to the proof of \cref{thm:not-SGC}
  \*[
    \EE \sup_{h\in \HS} \abs{\Risk{\Dist}{h} - \EmpRisk{S}{h}}
      & \geq  \EE \abs{\Risk{\Dist}{\fh{d}^{(\bar S)}} - \EmpRisk{S}{\fh{d}^{(\bar S)}}}
      \geq \EE \EmpRisk{S}{\fh{d}^{(\bar S)}} - \EE \Risk{\Dist}{\fh{d}^{(\bar S)}} \\
      & \geq 1 - \rbra{\binom{\sgcSize{d}}{2}+\sgcSize{d}} 2^{-2 d }
      = 1 - \binom{\sgcSize{d}+1}{2} 2^{-2 d }
  \]
\end{proof}

\begin{proof}[Proof of \cref{lem:Q-risk}]

For each $v\in \set{0,1}^{k_d}$ let $g_{k_d}(v,x) = (v_1,\dots v_{k_d},x_{k_d+1},\dots x_{2 d })$, and let $G_{k_d}(V,S_X) = (g_{k_d}(v_1,X_1),\dots ,g_{k_d}(v_{k_d},X_{\sgcSize{d}}))$

Starting with $L_S(Q)$.
\*[
  L_S(Q)
    & = \frac{1}{n}\sum_{i=1}^n (2^{-k})^n\sum_{V\in (\set{0,1}^k)^n} \ind{X_i \in \overline{G_k(V,S_X)}}\ind{X_i \not\in G_k(V,S_X)}\\
    & \leq  \frac{2^{-n k}}{n}\sum_{i=1}^n\sum_{V\in \set{0,1}^{n k}} \ind{X_i \in \overline{G_k(V,S_X)}} \ind{X_i[1:k]\neq v_i} .
  \]
  Now, $X_i$ is in $\overline{G_k(V,S_X)}$ only if $X_i[1:k]$ is antipodal to at least one $v$, so
\*[
\frac{2^{-n k}}{n}\sum_{i=1}^n\sum_{V\in \set{0,1}^{n k}} \ind{X_i \in \overline{G_k(V,S_X)}} \ind{X_i[1:k]\neq v_i}
& \leq  \frac{2^{-n k}}{n}\sum_{i=1}^n \sum_{V\in \set{0,1}^{n k}} \sum_{j=1}^n \ind{X_i[1:k] = \ol{v_j} \neq v_i}\ind{X_i[k+1 : 2 d ] = \ol{X_j[k+1:2 d ]}} \\
& =  \frac{2^{-n k}}{n}\sum_{i=1}^n \sum_{j=1}^n \ind{X_i[k+1 : 2 d ] = \ol{X_j[k+1:2 d ]}} \sum_{V\in \set{0,1}^{n k}} \ind{X_i[1:k] = \ol{v_j} \neq v_i}.
\]
  The inner most summand is constant in all but the $i$th and $j$th $v$, so
  \*[
  & \hspace{-2em}\frac{2^{-n k}}{n}\sum_{i=1}^n \sum_{j=1}^n \ind{X_i[k+1 : 2 d ] = \ol{X_j[k+1:2 d ]}} \sum_{V\in \set{0,1}^{n k}} \ind{X_i[1:k] = \ol{v_j} \neq v_i}\\
    & \leq  \frac{2^{-n k}}{n}\sum_{i=1}^n \sum_{j=1}^n \ind{X_i[k+1 : 2 d ] = \ol{X_j[k+1:2 d ]}}  \sum_{(v_i,v_j)\in \set{0,1}^{2 k}} 2^{k(n-2)} \ind{X_i[1:k] = \ol{v_j} \neq v_i} \\
    & =  \frac{2^{-2 k}}{n}\sum_{i=1}^n \sum_{j=1}^n \ind{X_i[k+1 : 2 d ] = \ol{X_j[k+1:2 d ]}}  \sum_{(v_i,v_j)\in \set{0,1}^{2 k}} \ind{X_i[1:k] = \ol{v_j} \neq v_i} \\
    & =  \frac{2^{-2 k}}{n}\sum_{i=1}^n \sum_{j=1}^n \ind{X_i[k+1 : 2 d ] = \ol{X_j[k+1:2 d ]}}  \sum_{v_j\in \set{0,1}^{k}} (2^k-1) \ind{X_i[1:k] = \ol{v_j}} \\
    & =  \frac{2^{-2 k}}{n}\sum_{i=1}^n \sum_{j=1}^n (2^k-1) \ind{X_i[k+1 : 2 d ] = \ol{X_j[k+1:2 d ]}}  \\
    & = \frac{2^{-k}(1- 2^{-k})}{n} \card{\set{(i,j)\in\range{n}^2 : X_i[k+1 : 2 d ] = \ol{X_j[k+1:2 d ]}}} .
\]
Thus
\*[
  L_S(Q) \leq \frac{2^{-k}(1- 2^{-k})}{n} \card{\set{(i,j)\in\range{n}^2 : X_i[k+1 : 2 d ] = \ol{X_j[k+1:2 d ]}}}.
\]
Taking expectations, we get
\*[
  \EE L_S(Q) \leq \frac{2^{-k}(1- 2^{-k})}{n} n(n-1)2^{-(2 d -k)} = (n-1)2^{-2 d } (1-2^{-k}).
\]

Next, Looking at $L_{\ol S}(Q)$,
\*[
  L_{\ol S}(Q)
    & = \frac{1}{n}\sum_{i=1}^n (2^{-k})^n\sum_{V\in (\set{0,1}^k)^n} \ind{\ol X_i \in \overline{G_k(V, S_X)}}\ind{\ol X_i \not\in G_k(V, S_X)}\\
    & \leq  \frac{2^{-n k}}{n}\sum_{i=1}^n\sum_{V\in \set{0,1}^{n k}} \ind{\ol X_i \in \overline{G_k(V,S_X)}}.
  \]
  Now, $\ol X_i$ is in $\overline{G_k(V,S_X)}$ only if $\ol {X_i[1:k]}$ is antipodal to at least one $v$, so
\*[
\frac{2^{-n k}}{n}\sum_{i=1}^n\sum_{V\in \set{0,1}^{n k}} \ind{\ol X_i \in \overline{G_k(V,S_X)}}
& \leq  \frac{2^{-n k}}{n}\sum_{i=1}^n \sum_{V\in \set{0,1}^{n k}} \sum_{j=1}^n \ind{\ol{X_i[1:k]} = \ol{v_j}}\ind{\ol{X_i[k+1 : 2 d ]} = \ol{X_j[k+1:2 d ]}} \\
& =  \frac{2^{-n k}}{n}\sum_{i=1}^n \sum_{j=1}^n \ind{\ol{X_i[k+1 : 2 d ]} = \ol{X_j[k+1:2 d ]}} \sum_{V\in \set{0,1}^{n k}} \ind{\ol{X_i[1:k]} = \ol{v_j}} \\
& =  \frac{2^{-n k}}{n}\sum_{i=1}^n \sum_{j=1}^n \ind{{X_i[k+1 : 2 d ]} = {X_j[k+1:2 d ]}} \sum_{V\in \set{0,1}^{n k}} \ind{{X_i[1:k]} = {v_j}} \\
.
\]
  The inner most summand is constant in all but the $j$th $v$, so
  \*[
  & \hspace{-2em}\frac{2^{-n k}}{n}\sum_{i=1}^n \sum_{j=1}^n \ind{X_i[k+1 : 2 d ] = {X_j[k+1:2 d ]}} \sum_{V\in \set{0,1}^{n k}} \ind{X_i[1:k] = {v_j}}\\
    & \leq  \frac{2^{-n k}}{n}\sum_{i=1}^n \sum_{j=1}^n \ind{X_i[k+1 : 2 d ] = {X_j[k+1:2 d ]}}  \sum_{v_j\in \set{0,1}^{k}} 2^{k(n-1)} \ind{X_i[1:k] = {v_j}} \\
    & =  \frac{2^{-k}}{n}\sum_{i=1}^n \sum_{j=1}^n \ind{X_i[k+1 : 2 d ] = {X_j[k+1:2 d ]}}  \sum_{v_j\in \set{0,1}^{k}} \ind{X_i[1:k] = {v_j}} \\
    & =  \frac{2^{-k}}{n}\sum_{i=1}^n \sum_{j=1}^n \ind{X_i[k+1 : 2 d ] = {X_j[k+1:2 d ]}}  \\
    & = \frac{2^{-k}}{n} \card{\set{(i,j)\in\range{n}^2 : X_i[k+1 : 2 d ] = {X_j[k+1:2 d ]}}} .
\]
Thus
\*[
  L_{\ol S} (Q) \leq \frac{2^{-k}}{n} \card{\set{(i,j)\in\range{n}^2 : X_i[k+1 : 2 d ] = {X_j[k+1:2 d ]}}}.
\]
Taking expectations, we get
\*[
  \EE L_{\ol S(Q)} \leq \frac{2^{-k}}{n} (n+n(n-1)2^{-(2 d -k)}) = 2^{-k} + (n-1)2^{-2 d }.
\]

Lastly, for $\Risk{\Dist}{Q}$,

  \*[
    \Risk{\Dist}{Q}
      & = \EE \Risk{\Dist}{\fh{d}}
       \leq n 2^{-2 d }
  \]
\end{proof}

\begin{proof}[Proof of \cref{lem:Q-sgc}]
First, using a standard symmetrization argument, we bound the supremum over the empirical process of Curryed losses by twice the Rademacher complexity,
  \s*[
    \EE \sbra{\sup_{S'} \abs{\Risk{\Dist}{Q(S')} -\Risk{S}{Q(S')}}}
      & \leq 2\ R(A)
  \s]
  Where
  \s*[A = \set{(\loss'(Z_i,Q(\fh{d}^{(S')})))_{i\in\range{\sgcSize{d}}} : S' \in (\sgcInstances{d})^{\sgcSize{d}}, \loss'\in\set{\loss,-\loss}}. \s]
  and $R(A)$ denotes the Rademacher complexity of $A$.
  Then, by Massart's Lemma,
  \*[
    R(A)
      & \leq \max_{a\in A}{\norm{a}} \frac{\sqrt{\log(2\card{A})}}{\sgcSize{d}} \\
      & \leq \max_{a\in A}{\norm{a}} \frac{\sqrt{((2 d -k_d)\sgcSize{d}+1)\log(2)}}{\sgcSize{d}} .
  \]
  Next, by the same arguments as in the proof of \cref{lem:Q-risk}, $\loss'(Z_i,Q(\fh{d}^{(S')})) \leq \sgcSize{d} 2^{-k_d}$ for all $i$ and for all datasets $S'$.
  Therefore, for all $a\in A$, $\norm{a}\leq \sgcSize{d}^{1/2} \sgcSize{d} 2^{-k_d} = \sgcSize{d}^{3/2} 2^{-k_d}$.

  Hence $R(A)\leq \sqrt{\log(2)}\sgcSize{d}^{1/2}((2 d -k_d)\sgcSize{d}+1)^{1/2}2^{-k_d}$.
\end{proof}

\begin{proof}[Proof of \cref{thm:example1-main-bound}]
  From \cref{surrogatecond,supbound,lem:Q-risk},
  \*[\label{eq:spec-risk-bound1}
    & \EE[\Risk{\Dist}{\fh{d}} - \EmpRisk{S}{\fh{d}}] \\
      & \quad \le \EE[ \abs{\EmpRisk{S}{Q} - \EmpRisk{S}{\fh{d}}} ]\\
      &\qquad\quad +  \,\Pr[Q \not\in \sgcSHS{d}] + \EE \bigg[ \sup_{P \in \sgcSHS{d}} { \Risk{\Dist}{P} - \EmpRisk{S}{P} } \bigg ] \\
      & \quad \le (\sgcSize{d}-1)2^{-2 d } (1-2^{-k_d})+ 0 + 0 \\
      &\qquad\quad + \EE \bigg[ \sup_{P \in \sgcSHS{d}} { \Risk{\Dist}{P} - \EmpRisk{S}{P} } \bigg ] .
  \]
  The last term is controlled using \cref{lem:Q-sgc} to get
  \[\label{eq:spec-risk-bound2}
    & \EE[\Risk{\Dist}{\fh{d}} - \EmpRisk{S}{\fh{d}}] \\
      &\quad \leq (\sgcSize{d}-1)2^{-2 d } (1-2^{-k_d}) \\
      &\qquad\quad +  2\sqrt{\log(2)}\sgcSize{d}^{1/2}((2 d -k_d)\sgcSize{d}+1)^{1/2}2^{-k_d}\\
      &\quad \leq  (\sgcSize{d}-1)2^{-2 d }\\
      &\qquad\quad + 2\sqrt{\log(2)}\sgcSize{d}^{1/2}((2 d -k_d)\sgcSize{d}+1)^{1/2}2^{-k_d} .
  \]
\end{proof}

\section{Proofs for Overparameterized Linear Regression}
\label{apx:Proofs2}

\begin{proof}[Proof of \cref{lem:lr-failureUC}]
Let $\phi((x,y)) = (x,2 x\beta -y)$. Then $(x,y)$ and $(\phi(x,y))$ are equally probable, and $\phi$ is its own inverse function. Hence $\phi$ is measure preserving. Moreover $L_S(\hat\beta(\phi(X,Y))) = \frac{4}{n}\norm{Z}^2$. Hence for any sequence of sets $A_n\subset \dataspace^n$
\*[
  & \EE \sup_{(\tilde X,\tilde Y)\in A_n} \abs{L_D(\hat\beta(\tilde X,\tilde Y)) -L_S(\hat\beta(\tilde X,\tilde Y))}\\
  & \quad \geq \EE \one_{\phi(X,Y)\in A_n} \abs{L_D(\hat\beta(\phi(X,Y))) -L_S(\hat\beta(\phi(X,Y)))}\\
  & \quad \geq \EE \one_{\phi(X,Y)\in A} \max\rbra{0, 4\frac{\norm{Z}^2}{n} - L_D(\hat\beta(\phi(X,Y)))}\\
\]

We can couple the spaces for different values of $n$ in any way we choose since no terms in the statement involve multiple probability spaces at once. Hence we can do so in a way that $\Pr(\set{\forall n\in\Nats:\ \phi(X,Y)\in A_n}) \geq 2/3$.

Using \citet[Theorem 4]{bartlett2019benign} and the weak law of large numbers, we have that $L_D(\hat\beta(\phi(X,Y))\stk\to{P}\sigma^2$ and  $4\frac{\norm{Z}^2}{n}\stk\to{P} 4\sigma^2$ when $\Sigma_n$ is benign. Then there is a subsequence along which this convergence is almost sure. Along this subsequence, \[\max\rbra{0, 4\frac{\norm{Z}^2}{n_k} - L_D(\hat\beta(\phi(X,Y)))} \stk\to{a.s.} 3\sigma^2.\]
Then by Fatou's lemma along the subsequence
\*[
  & \liminf_{k\to\infty}\ \EE \sup_{(\tilde X,\tilde Y)\in A_{n_k}} \abs{L_D(\hat\beta(\tilde X,\tilde Y)) -L_S(\hat\beta(\tilde X,\tilde Y))} \\
  &\quad \geq \EE \liminf_{k\to\infty}\one_{\phi(X,Y)\in A} 3\sigma^2\\
  &\quad \geq 2\sigma^2
\]
Thus there is a sub-subsequence above $2\sigma^2-\epsilon$ infinitely often for each $\epsilon>0$, and hence
\s*[
  \limsup_{n\to\infty}\ \EE \sup_{(\tilde X,\tilde Y)\in A_n} \abs{L_D(\hat\beta(\tilde X,\tilde Y)) -L_S(\hat\beta(\tilde X,\tilde Y))}
  &\geq 2\sigma^2.
\s]
\end{proof}

\begin{proof}[Proof of \cref{lem:LR-sur-decomp}]
  Let $\hat \beta_0(X_0)=(X_0'X_0)^+X_0'X_0\beta =P(X_0)\beta$. This corresponds to the classifier that solves the learning problem without label noise if the training design matrix was $X_0$. Let $\hat \beta_0 =\hat\beta_0(X)$ (if no argument is specified then it is the ``learned'' version ).

  Let the projection onto the row-span of a matrix $A$ be given by $P(A) = (A'A)^+ A'A  =  A'A (A'A)^+ = A' (AA')^+ A$.

  If $d>n$ then (a.s.) $X$ is of rank $n$ so $P(X)$ is a rank $n$ projection on $\Reals^d$ and $P(X')$ is a rank $n$ projection on $\Reals^n$ --- so $P(X')=I$ a.s.

  The surrogate decomposition (not in expectation) gives us
  \[
    L_D(\hat \beta ) - L_S(\hat \beta)
      & = L_S(\hat \beta_0)-L_S(\hat\beta) + (L_D(\hat \beta) - L_D(\hat \beta_0)) + (L_D(\hat \beta_0) - L_S(\hat \beta_0)).
  \]

Then,
  \[
    L_S(\hat \beta_0)
      & = \frac{1}{n}\norm{X(X'X)^+X'X\beta_n - (X\beta_n+Z)}^2 \\
      & = \frac{1}{n}\norm{-Z +(P(X')-I)X\beta_n}^2 \\
      & = \frac{1}{n}\norm{Z}^2.
  \]

Next, for any $X_0$,
  \[ \label{eq:lr-pf-surr-class}
    L_D(\hat \beta_0(X_0)) - L_S(\hat \beta_0(X_0))
      & = \EE_{x,z}\norm{x P(X_0)\beta_n - (x\beta_n+z)}^2
            - \frac{1}{n}\norm{X P(X_0)\beta_n - (X\beta_n+Z)}^2\\
      & = \EE_{x,z}\norm{x P(X)\beta_n - (x\beta_n+z)}^2
            - \frac{1}{n}\norm{X P(X_0)\beta_n - (X\beta_n+Z)}^2\\
      & = \EE_{x,z}\norm{x P(X_0)\beta_n - x\beta_n}^2 +\sigma^2 \\
            &\qquad - \frac{1}{n}\norm{X P(X_0)\beta_n - X\beta_n}^2 +\frac{2}{n} Z'(X (I-P(X_0))\beta_n) -\frac{1}{n}\norm{Z}^2
            \\
      & =\sigma^2-\frac{1}{n}\norm{Z}^2+\frac{2}{n} Z'(X (I-P(X_0))\beta_n) \\
      &\qquad + \EE_{x,z}\norm{x[P(X_0) - I]\beta_n}^2
            - \frac{1}{n}\norm{X[ P(X_0)-I]\beta_n}^2\\
      & = \sigma^2-\frac{1}{n}\norm{Z}^2+\frac{2}{n} Z'(X (I-P(X_0))\beta_n)\\
        & \qquad + \beta_n' \sbra{[ P(X_0)-I] \sbra{\Sigma_n - \frac{1}{n}X'X} [ P(X_0)-I] } \beta_n .
  \]
Since $X P(X) = X$, then when $X_0=X$, this simplifies to
\[
  L_D(\hat \beta_0(X)) - L_S(\hat \beta_0(X))
    & = \sigma^2-\frac{1}{n}\norm{Z}^2+\\
      & \qquad + \beta_n' \sbra{[ P(X)-I] \Sigma_n [ P(X)-I] } \beta_n .
\]

Lastly,
\[
    L_D(\hat \beta) - L_D(\hat \beta_0)
    & = \EE_{x,z}\rbra{x(X'X)^+X'(X\beta_n+Z) - (x\beta_n+z)}^2
          - \EE_{x,z}\rbra{x(X'X)^+X'X\beta_n - (x\beta_n+z)}^2\\
    & = \EE_{x}\rbra{x(X'X)^+X'(X\beta_n+Z) - x\beta_n}^2
          - \EE_{x}\rbra{x(X'X)^+X'X\beta_n - x\beta_n}^2\\
    & = \EE_{x}\rbra{x(X'X)^+X'Z}^2\\
    & = \EE_{x} (Z'X(X'X)^+x'x(X'X)^+X'Z)\\
    & = Z'X(X'X)^+\Sigma(X'X)^+X'Z\\
    & = \tr(X(X'X)^+\Sigma(X'X)^+X'Z Z')\\
\]
\end{proof}

\begin{proof}[Proof of \cref{lem:lr-sgcsurr}]
  We want to bound
  \[
    &\EE\sup_{X_0} \abs{L_D(\hat \beta_0(X_0)) - L_S(\hat \beta_0(X_0))}\\
      & = \EE \sup_{X_0}\abs{\sigma^2-\frac{1}{n}\norm{Z}^2+\frac{2}{n} Z'(X (I-P(X_0))\beta_n)+\beta_n' \sbra{[ P(X_0)-I] \sbra{\Sigma_n - \frac{1}{n}X'X} [ P(X_0)-I] } \beta_n} \\
      &\leq \EE\abs{\sigma^2-\frac{1}{n}\norm{Z}^2} +\EE \sup_{X_0}\abs{\frac{2}{n} Z'(X (I-P(X_0))\beta_n)}+\EE\sup_{X_0}\abs{\beta_n' \sbra{[ P(X_0)-I] \sbra{\Sigma_n - \frac{1}{n}X'X} [ P(X_0)-I] } \beta_n} \\
  \]
We will handle each of the three terms separately.

First
\[
  \EE\abs{\sigma^2-\frac{1}{n}\norm{Z}^2}
    & \leq \sqrt{\EE\sbra{\sigma^2-\frac{1}{n}\norm{Z}^2}^2} \\
    & = \sqrt{\EE\sbra{\sigma^4 - \frac{2\sigma^2}{n}\norm{Z}^2+\frac{1}{n^2}\sum_{i\in\range{n}}\sum_{j\in\range{n}} Z_i^2 Z_j^2}}\\
    & = \sqrt{\sigma^4 - 2\sigma^4+\frac{n(n-1)}{n^2}\sigma^4 + 3\frac{n}{n^2}\sigma^4}\\
    & = \sigma^2\sqrt{\frac{2}{n}} .
\]

Next, for some universal constants $C_2,C_3>0$,
the second term is bounded by $C_2\frac{ \sigma\norm{\beta_n}\norm{\Sigma_n}^{1/2} \sqrt{r_0(\Sigma_n)}}{\sqrt{n}}$ in \cref{lem:lr-second-term-uc}
and the third term is bounded by $C_3 \norm{\beta_n}^2\norm{\Sigma_n} \max\rbra{\sqrt{\frac{r_0(\Sigma_n)}{n}}, {\frac{r_0(\Sigma_n)}{n}}}$ in \cref{lem:lr-third-term-uc}.

Putting these all together yields the desired result.
\end{proof}

\begin{lemma}[Third Term]
\label{lem:lr-third-term-uc}
For some universal constant $C_3>0$
\[
&\hspace{-1em}
 \EE\sup_{X_0\in\Reals^{n\times d}} \abs{\beta_n' [P(X_0)-I] \sbra{\Sigma_n - \frac{1}{n}X'X} [P(X_0)-I]  \beta_n} \\
      & =   \EE\sup_{P\in\OrthProjs(d,d-n)} \abs{\beta_n'P \sbra{\Sigma_n - \frac{1}{n}X'X} P \beta_n} \\
      & \leq C_3 \norm{\beta_n}^2\norm{\Sigma_n} \max\rbra{\sqrt{\frac{r_0(\Sigma_n)}{n}}, {\frac{r_0(\Sigma_n)}{n}}}
\]
where $\OrthProjs(d,k)$ is the collection of orthogonal projections on $\Reals^d$ of rank $k$.
\end{lemma}

This follows directly from \citet[Theorem~4.]{koltchinskii2014concentration}.

\begin{lemma}[Second Term]
\label{lem:lr-second-term-uc}
For some universal constant $C_2>0$
\[
  \EE\sup_{X_0\in\Reals^{n\times d}} \abs{\frac{2}{n} Z' X P(X_0)\beta_n}
      & =   \EE\sup_{P\in\OrthProjs(d,d-n)} \abs{\frac{2}{n} Z' X P\beta_n}
      \leq C_2\frac{2 \sigma\norm{\beta_n}\norm{\Sigma_n}^{1/2}\sqrt{r_0(\Sigma_n)}}{\sqrt{n}}
\]
where $\OrthProjs(d,k)$ is the collection of orthogonal projections on $\Reals^d$ of rank $k$.
\end{lemma}

\begin{proof}[Proof of \cref{lem:lr-second-term-uc}]
  \[
    \sup_{P\in\OrthProjs(d,d-n)} \abs{\frac{2}{n} Z' X P\beta_n}
      & \leq \sup_{\norm{\gamma}\leq \norm{\beta_n}} \abs{\frac{2}{n} Z' X \gamma} \\
      & = \abs{\frac{2}{n} Z' X \frac{X' Z}{\norm{X'Z}} \norm{\beta_n}}\\
        & = \frac{2 \norm{\beta_n}}{n} \norm{X'Z}.
  \]
Now,
\[
  \EE \sup_{P\in\OrthProjs(d,d-n)} \abs{\frac{2}{n} Z' X P\beta_n}
    & \leq \frac{2 \norm{\beta_n}}{n} \EE \norm{X'Z} \\
    & \leq \frac{2 \norm{\beta_n}}{n} \sqrt{\EE Z' X X' Z}\\
    & = \frac{2 \norm{\beta_n}}{n} \sqrt{\EE \tr( X' ZZ' X)}\\
    & = \frac{2 \norm{\beta_n}}{n} \sqrt{\EE \tr( X' \sigma^2 I X)}\\
    & = \frac{2 \norm{\beta_n}\sigma}{n} \sqrt{\EE \tr(XX')}\\
    & = \frac{2 \norm{\beta_n}\sigma}{n} \sqrt{n \tr(\Sigma_n)}\\
    & = \frac{2 \sigma\norm{\beta_n}\sqrt{\tr(\Sigma_n)}}{\sqrt{n}}\\
    & \leq \frac{2 \sigma\norm{\beta_n}\norm{\Sigma_n}\sqrt{r_0(\Sigma_n)}}{\sqrt{n}}
\]

\end{proof}

\begin{proof}[Proof of \cref{thm:lr-exp-risk-bd}]
  We only need bounds on $\EE(L_S(\hat\beta_0)-L_S(\hat\beta))$ and $\EE(L_D(\hat\beta)-L_D(\hat\beta_0))$ to combine with  \cref{lem:LR-sur-decomp} and \cref{lem:lr-sgcsurr}.

  First,
  \[
  \EE(L_S(\hat\beta_0)-L_S(\hat\beta)) = \EE \frac{\norm{Z}^2}{n} = \sigma^2
  \]

  Second, \citet{bartlett2019benign} shows that there are universal constant, $c,b>0$, such that for all $\delta <1$, with probability at least $(1-\delta)$,
  \[
  L_D(\hat \beta) - L_D(\hat \beta_0)
    & = Z'X(X'X)^+\Sigma(X'X)^+X'Z \\
    & \leq c \sigma^2\log(1/\delta)\rbra{\frac{k^*}{n} + \frac{n}{R_{k^*}(\Sigma_n)}} ,
  \]
  where $k_* = \min\set{k\geq 0: r_k(\Sigma_n)\geq b n}$, $r_k(\Sigma_n) = \frac{\sum_{i>k} \lambda_i(\Sigma_n)}{\lambda_{k+1}(\Sigma_n)}$ and $R_k(\Sigma_n) = \frac{\rbra{\sum_{i>k}\lambda_i(\Sigma_n)}^2}{\sum_{i>k}\lambda_i^2(\Sigma_n)}$.

  We can turn this into a bound in expectation by integrating the tail.
  \[
    \frac{\EE L_D(\hat \beta) - L_D(\hat \beta_0) }{c\sigma^2 \rbra{\frac{k^*}{n} + \frac{n}{R_{k^*}(\Sigma_n)}}}
      & = \int_0^\infty \Pr\rbra{\frac{L_D(\hat \beta) - L_D(\hat \beta_0) }{c \sigma^2\rbra{\frac{k^*}{n} + \frac{n}{R_{k^*}(\Sigma_n)}}} >t} d t \\
      & = \int_0^\infty e^{-t} d t \\
      & = 1 .
  \]
  Thus, for the same universal constants, $c,b$,
  \[
  \EE L_D(\hat \beta) - L_D(\hat \beta_0) \leq c \sigma^2\rbra{\frac{k^*}{n} + \frac{n}{R_{k^*}(\Sigma_n)}} .
  \]
\end{proof}

\end{document}